\documentclass[10pt,conference]{IEEEtran}
\usepackage{cite}
\usepackage{amsmath,amssymb,amsfonts}
\usepackage{graphicx}
\usepackage{textcomp}
\usepackage{hyperref}
\usepackage{xcolor}
\usepackage[utf8]{inputenc}
\usepackage{bm}
\usepackage{amsmath}
\usepackage{geometry}
\usepackage{amsthm}
\newtheorem{theorem}{Theorem}
\newtheorem{lemma}{Lemma}
\newtheorem{definition}{Definition}
\newtheorem{game}{Game}

\usepackage{comment}
\usepackage{color}
\usepackage{algorithm}  
\usepackage{algorithmicx}  
\usepackage{algpseudocode} 
\usepackage{setspace}
\newtheorem{assumption}{Assumption}
\newtheorem{proposition}{Proposition}

\newtheorem{question}{Question}

\def\N{\mathcal{N}}
\def\x{\boldsymbol{x}}
\def\D{\mathcal{D}}

\def\xb{\boldsymbol{x}}
\def\xbswm{\xb^{\textsc{swm}}}
\def\xswm{x^{\textsc{swm}}}
\def\xbne{\xb^{\textsc{ne}}}
\def\xne{x^{\textsc{ne}}}

\def\xz{x^{\circ \textsc{swm}}}

\def\xzne{x^{\circ \textsc{ne}}}
\def\hne{h^{\textsc{ne}}}
\def\hswm{h^{\textsc{swm}}}

\ifodd 1

\newcommand{\revw}[1]{\textcolor{red}{#1}}

\newcommand{\com}[1]{\textbf{\color{red} \left(Comment: #1\right) }}
\newcommand{\comg}[1]{\textbf{\color{blue} \left(COMMENT: #1\right)}}
\newcommand{\response}[1]{\textbf{\color{blue} \left(RESPONSE: #1\right)}}
\else

\newcommand{\revw}[1]{#1}
\newcommand{\com}[1]{}
\newcommand{\comg}[1]{}
\newcommand{\response}[1]{}
\fi

\IEEEoverridecommandlockouts\IEEEpubid{\makebox[\columnwidth]{ 979-8-3503-1090-0/23/\$31.00~\copyright~2023 IEEE \hfill} \hspace{\columnsep}\makebox[\columnwidth]{ }}

\begin{document}

\title{Price of Stability in Quality-Aware Federated 
Learning\vspace{-2mm}}

\author{\IEEEauthorblockN{Yizhou Yan$^{\dag}$$^*$, Xinyu Tang$^{\dag}$$^*$, Chao Huang$^{\ddag}$, and Ming Tang$^{\dag}$}
\IEEEauthorblockA{$^{\dag}$Department of Computer Science and Engineering, Southern University of Science and Technology \\
$^{\ddag}$Department of Computer Science,
The University of California, Davis
\vspace{-2mm}}
\thanks{This work was supported in part by Guangdong Basic and Applied Basic Research Foundation under Grant 2023A1515012819 and the National Natural Science Foundation of China under Grant 62202214. (Corresponding author: Chao Huang) $^*$Both authors contributed equally to this research.}
}

\maketitle

\begin{abstract}
Federated Learning (FL) is a distributed machine learning scheme that enables clients to train a shared global model without exchanging local data. The presence of label noise can severely degrade the FL performance, and  some existing studies have focused on algorithm design for label denoising. However, they ignored the important issue that clients may not apply costly label denoising strategies due to them being self-interested and having heterogeneous valuations on the FL performance. To fill this gap, we model the clients' interactions as a novel label denoising game and characterize its equilibrium. 
We also analyze the price of stability, which quantifies the difference in the system performance (e.g., global model accuracy, social welfare) between the equilibrium outcome and the socially optimal solution. We prove that the equilibrium outcome always leads to a lower global model accuracy than the socially optimal solution does. We further design an efficient algorithm to compute the socially optimal solution. 
Numerical experiments on MNIST dataset show that the price of stability increases as the clients' data become noisier, calling for an effective incentive mechanism. 
\end{abstract}

\section{Introduction}
Federated learning (FL) is an emerging machine learning paradigm that enables distributed clients (e.g., devices,  organizations) to collaboratively train a global model while keeping their data local \cite{kairouz2021advances}. This approach offers several advantages over traditional centralized machine learning, such as improved data privacy and the ability to leverage distributed computational resources. Thus, FL is a promising candidate approach to improving the performance of networking systems. This can be achieved by letting network entities (e.g., devices, operators) act as clients and cooperatively learn system characteristics (e.g., user behaviors, network environment) for prediction and adaptive control.

An FL process typically follows an iterative procedure that consists of several rounds of communication between the clients and a trusted central server. In each round, the server sends the current global model to the clients. These clients then use their local data to compute model updates, e.g., gradients or parameter changes, and send these updates back to the server. The server updates the global model by aggregating the received updates (e.g., weighted averaging). This process repeats until the global model converges \cite{ke2022quantifying}.


The performance of the global model crucially depend on data heterogeneity and data quality among clients: 
\begin{itemize}
\item \textit{Data heterogeneity}: It pertains to the differences in data distribution among clients. It has been shown to severely compromise the FL convergence and generalization.  Existing studies have concentrated on addressing data heterogeneity by using techniques such as regularization and personalization \cite{tan2022towards}. 
\item \textit{Data quality}: It refers to the conditions of the training data such as the noise rate of labels. Label noise can arise from various sources, e.g.,  mislabeling, human errors during data collection and annotation. Take user mobility traces as an example. Indoor mobile devices connecting to WiFi tend to have better localization accuracy than those devices without WiFi connection do. This leads to a training dataset with higher quality
(i.e., fewer noisy/wrong labels). Label noise is a common challenge in real-world data \cite{fang2022robust}, but it is under-explored in FL and hence the focus of our paper. 
\end{itemize} 

Some recent studies have focused on the label noise detection and label denoising (i.e., correct the wrong labels) \cite{fang2022robust,xu2022fedcorr, zeng2022clc}. For example, Xu \textit{et. al} in \cite{xu2022fedcorr} proposed FedCorr, a multi-stage framework that dynamically identifies and corrects incorrect labels based on per-sample losses. Zeng \textit{at. al} in \cite{zeng2022clc} proposed a consensus-based approach to detect label noise. 
These recent advances have presented FL clients with the opportunity to provide higher quality data, subsequently improving model performance. However, it is important to recognize that in practice, clients (e.g., devices, companies) are self-interested and heterogeneous in terms of their valuations on  model performance \cite{tang2021incentive}. Meanwhile, denoising is a resource-intensive process, often requiring significant  computing power and/or manpower \cite{xu2022fedcorr}. Given these constraints, it is imperative for clients to optimize their denoising strategies, striking a balance between the expected improvement in model performance and the associated denoising costs. This motivates our first key question:
\begin{question}
How will heterogeneous self-interested clients decide their label denoising strategies?
\end{question}
To answer \textit{Question 1}, we model the interactions among heterogeneous clients as a non-cooperative game, in which each client decides the denoising strategy to maximize a tradeoff between the global model performance and the denoising cost. Solving (and even formulating) the game is challenging, as it is unclear how label noise affects FL. To this end, we consider two types of label noise below \cite{gu2021instancedependent}.
\begin{itemize}
\item \textit{Random flipping}: It refers to a case where a label is randomly flipped to another class. Random flipping can model where the annotators have limited expertise or when the annotation process is subject to random errors.
\item \textit{Instance-dependent noise}: Unlike random flipping, instance-dependent noises are not uniformly distributed across the dataset. Instead, the probability of mislabeling depends on the characteristics of the instances themselves. This type of noise can model scenarios where some instances are inherently more difficult to annotate correctly or are more prone to ambiguity.
\end{itemize}
We  conduct experiments to evaluate how these two noise types affect the global model performance, which helps with the game formulation. We characterize the game equilibrium and propose an efficient algorithm to compute it. 

It is known that in many situations, the game equilibrium can be inefficient, leading to suboptimal outcomes for the entire system. This inefficiency arises because clients aim to maximize their individual payoffs, often without considering the collective welfare of the system \cite{donahue2021optimality}. This motivates the study of the price of stability (PoS), which measures the efficiency loss resulting from non-cooperative behavior compared to a socially optimal solution. 
\begin{question}
What is the PoS in quality-aware FL?
\end{question}
Answering \textit{Question 2} is not easy either in the context of quality-aware FL. Specifically, due to the coupling of the clients' denoising decisions and their multi-dimensional heterogeneity (i.e.,  valuations and data quality/noise rate), it is difficult to determine the equilibrium outcome and socially optimal solution in closed-from. Nonetheless, we prove the performance degradation of the equilibrium outcome in terms of the global model accuracy. Moreover, we propose an efficient algorithm to compute the socially optimal solution and empirically evaluate the PoS.

The key contributions of this paper are summarized below.

\begin{itemize}
\item To our best knowledge, we present the first work that studies the PoS in quality-aware FL.  Our study facilitates understanding the extent to which the system's performance degrades due to non-cooperative behavior and helps identify opportunities to improve the overall efficiency in practical FL systems.
\item We derive the socially optimal solution and equilibrium outcome of quality-aware FL and prove their uniqueness. We prove that the equilibrium outcome always leads to a higher average noise rate and a lower global model accuracy than the socially optimal solution does. We further design an efficient algorithm to obtain the socially optimal solution for empirical PoS evaluation.  
\item We conduct case studies to investigate two types of label noise: random flipping and instance-dependent noise. 
Numerical experiments show that random flipping is more harmful than instance-dependent noise to the global model accuracy.  Moreover, we find that PoS increases as the clients' data become noisier, calling for an effective incentive mechanism.  

\end{itemize}


\section{System Model}

We first present the quality-aware FL model and characterize the clients' payoff. Then, we formulate the social welfare maximization problem and the non-cooperative game.

\subsection{Quality-Aware FL Model}
We consider $N$ clients, denoted by $\N=\{1,2,\ldots, N\}$. Each client $n\in\N$ has a local dataset $\D_n$ with $d_n$ data samples. 
We focus on supervised learning tasks in FL, where each data sample consists of a set of features and a label. These clients cooperate in training a global model that maps from feature to label with their local datasets. We consider a scenario where the utility of each client relies on the accuracy of the global model, e.g.,  network operators can use the global model for product improvement to earn revenues.

There may exist label noise/errors in clients' datasets, i.e., the labels of some data samples may be incorrect. Let $\epsilon_n$ denote the noise rate of client $n$'s dataset, i.e., the fraction of client $n$'s data samples that have incorrect labels. Suppose clients are able to rectify incorrect labels with some costs \cite{xu2022fedcorr}. Let decision variable $x_n\in [0, \epsilon_n]$  denote the noise rate of client $n$'s dataset after label correction, and $\xb\triangleq \left(x_n, n\in\N\right)$. Then, the average noise rate after correction among all clients is given by 
\begin{equation}
    \bar{x}(\x) = \frac{\sum_{n\in\N}d_n x_n}{\sum_{n'\in\N}d_{n'}}.
\end{equation}
After the FL process, the global model has an accuracy of 
\begin{equation}\label{eq:acc}
    A\left(\xb\right) = g \left( \bar{x}(\x)\right),
\end{equation}
where $g(\cdot)$ is a concave and decreasing function. Such assumptions on $g(\cdot)$ are supported by the experimental results under random flipping and instance-dependent noise in Section \ref{sec:case}. Intuitively, the average noise rate (rather than the noise rate of any arbitrary client) affects the accuracy of the global model. When the average noise rate is higher, the accuracy  is lower. Meanwhile, the marginal decrease of the accuracy is increasing in the average noise rate.  

From a social perspective, correcting labels can improve the global model accuracy. 
From a client's perspective, it incurs additional costs. Thus, the clients' individual interests may conflict with the goal of social welfare maximization. 


\subsection{Payoff of the Clients}
Given decision variable $\x$ and accuracy $A\left(\xb\right)$, the utility of client $n\in\N$ is given by
\begin{equation}
    U_n\left(A\left(\xb\right)\right) = w_n A\left(\xb\right),
\end{equation}
where $w_n>0$ is the unit revenue of client $n$. For example, when clients are network operators, $w_n$ is the revenue that the operator can earn per accuracy improvement \cite{tang2021incentive, zhang2022enabling}. 

Given  decision $x_n$, client $n$ has a cost for label correction: 
\begin{equation}\label{cost-function}
    C_n\left(x_n\right) = d_n\left(f\left(x_n\right) - f\left(\epsilon_n\right)\right).
\end{equation}
We assume that $f(\cdot)$ is a strictly convex and decreasing function. Intuitively, the label correction cost is  proportional to the size of the dataset $d_n$. Meanwhile, as the noise rate of the dataset decreases, the cost of reducing the same noise rate of the dataset increases. If the dataset has not been corrected (i.e., $x_n=\epsilon_n$), the cost is zero. We show a specific formulation of $f(\cdot)$ with practical insights in Section \ref{sec:case}.

The payoff function of a client $n\in\N$ is defined as the difference between its utility and cost:
\begin{equation}
    P_n\left(\xb\right) = U_n\left(A\left(\xb\right)\right) -  C_n\left(x_n\right).
\end{equation}

\subsection{Social Welfare Maximization and Game Formulation}
From a social perspective, clients should optimize the average noise rate after correction to maximize the social welfare. The socially optimal solution $\xbswm$  is given by
\begin{subequations}\label{eq:swm}
    \begin{align}
       \xbswm \triangleq\arg ~&  \underset{\xb}{\text{maximize}}
        & &  \textstyle \text{SW}(\xb)\triangleq \sum_{n\in\N}P_n\left(\xb\right) \label{eq:sw}\\
        & \text{subject to}
        & & x_n \in [0, \epsilon_n], ~ n \in \N.
    \end{align}
\end{subequations}

However, clients are usually self-interested. The interaction among clients is modeled as a non-cooperative game.
\begin{game}[Label Denoising Game]\label{game:error}~~

\begin{itemize}
    \item Players: clients $n \in \N$;
    \item Strategy: $x_n \in [0, \epsilon_n]$ for each client $n \in \N$;
    \item Payoff function: $P_n\left(\xb\right)$ for  each client $n \in \N$.
\end{itemize}

\end{game}
The stable state of the quality-aware FL system with self-interested clients corresponds to the situation under Nash equilibrium (NE) $\xbne$ of Game \ref{game:error}, i.e., the strategy profile such that no client has the incentive to deviate from. Let $\xb_{-n}\triangleq (x_{i}, i\in\N \setminus \{n\})$ denote the decisions of all clients except client $n\in\N$. For presentation simplicity, we use notations $P_n(\xb)$ and $P_{n}(x_n,\xb_{-n})$ interchangeably.
\begin{definition}[NE]
An NE  of Game \ref{game:error} is an $\xbne$ satisfying
\begin{equation}
    P_n\left(\xne_n, \xbne_{-n}\right) \geq P_n\left(x_n, \xbne_{-n}\right), ~ x_n \in [0, \epsilon_n], n \in \N.
\end{equation}
\end{definition}


We aim to analyze the label denoising strategies of  clients and determine the PoS of quality-aware FL systems. 


\section{Theoretical Analysis}\label{sec:analysis}
In this section, we first derive the socially optimal solution and the associated necessary condition, which provides practical insights and algorithm design guidance. Then, we derive the NE of Game \ref{game:error}. Finally, we characterize the PoS in model accuracy, which is challenging to analyze due to the coupled decision variables and client heterogeneity.

    
\subsection{Socially Optimal Solution}
Since problem \eqref{eq:swm} is a convex programming problem, we determine its optimal solution by checking the KKT conditions. Specifically, let $\hswm_n (x_n)\triangleq {\partial (\sum_{i\in\N}P_i\left(\xb\right))}/{\partial x_n}$ denote the partial derivative of the social welfare with respect to $x_n$. Let $\bar{w}\triangleq \sum_{i\in\N}w_{i}/N$ and $\bar{d}\triangleq \sum_{i\in\N}d_{i}/N$. Thus,
\begin{equation}\label{eq:fod}
        \hswm_n\left(x_n\right) = \Big(\frac{\bar{w}}{\bar{d}} \frac{\partial g\left(\bar{x}(\x)\right)}{\partial \bar{x}(\x)} - \frac{\partial f(x_n)}{\partial x_n}\Big)d_n,
\end{equation}
which is strictly decreasing in $x_n$ due to the strict convexity of $f(\cdot)$ and the concavity of $g(\cdot)$. We use $\xz_n$ to denote the unique zero point of $\hswm_n(x_n)$. The optimal solution to problem \eqref{eq:swm} is given  as follows.
\begin{lemma}[Optimal Solution]\label{lem:opt} The optimal solution $\xswm_n$ for $n\in\N$ to problem \eqref{eq:swm} is given by 
    \begin{equation}
\xswm_n = \left\{\begin{array}{ll}
    \epsilon_n, & \rm{if } \; \text{$\hswm_n\left(\epsilon_n\right) > 0$},  \\
    0, & \rm{if }\;  \text{$\hswm_n\left(0\right) < 0$}, \\
    \xz_n, & \rm{otherwise}.
    \end{array}\right.
\end{equation}
Further, $\xswm_n$ decreases in $\bar{w}$ and increases in $\bar{d}$. Given a fixed value of $\bar{d}$, $\xswm_n$ is independent of $d_n$. 
\end{lemma}
Lemma \ref{lem:opt} can be proven by analyzing the impact of system parameters  $\bar{w}$, $\bar{d}$, and $d_n$ on the zero point of $ \hswm_n(x_n) $ in \eqref{eq:fod}. The detailed proof is given in \cite{appendix}.

Lemma \ref{lem:opt} shows that if the average unit revenue of the clients (i.e., $\bar{w}$) is larger,  then each client $n$ should  choose a smaller noise rate $x_n$ (i.e., a higher data quality) to maximize the social welfare. Moreover, if the average number of data samples is smaller, then a smaller  noise rate $x_n$  should be chosen by the clients in order to maintain sufficient data samples with correct labels. However, given a fixed average number of data samples, the increase in the number of data samples of a client does not affect the optimal solution. 



According to Lemma \ref{lem:opt}, we further derive the necessary condition for the optimal solution $\xbswm$. The condition can provide additional insights and will be used for the algorithm design in Section \ref{subsec:alg-swm}. We first present a lemma to support our analysis, with the proof  given in \cite{appendix}.

\begin{lemma}\label{lem:con}
If $l(\cdot)$ is convex, then for any $a < b < c < d$, 
\begin{equation}\label{eq:h}
    {l\left(b\right) - l\left(a\right)}/{(b-a)} \leq {l\left(d\right) - l\left(c\right)}/{(d-c)}.
\end{equation}
If  $l(\cdot)$ is strictly convex, then the strict inequality holds in \eqref{eq:h}. If  $l(\cdot)$ is  concave or strictly concave, the inequality operator in \eqref{eq:h} should be replaced by $\geq$ or $>$, respectively.
\end{lemma}

Now, we can derive the necessary  condition for $\xbswm$.
\begin{proposition}[Necessary Condition for Optimal Solution]\label{them:sufficnt}
An $\xbswm$ is an optimal solution to problem \eqref{eq:swm}  only if $\xswm_n = \min\{\max_{i\in\N} \xswm_i, \epsilon_n\}$ for any $n\in\N$. That is, there exists a threshold $\theta\geq 0$ such that 
\begin{equation}\label{eq:theta1}
    \xswm_n = \left\{\begin{array}{ll}
    \epsilon_n, & \rm{if } ~ \text{$\epsilon_n < \theta$},\\
    \theta,& \rm {otherwise}.
    \end{array}\right.
\end{equation}
\end{proposition}
\begin{proof}
We prove by contradiction. Suppose there exists an $x_n$ that satisfies  $x_n < \epsilon_n$ and $x_n < x_m $ for an $m\in\mathcal{N}$ in  solution $\xb$. 
Then, the social welfare can be increased by replacing $x_n$ and $x_m$ with $x_n^{\prime}$ and 
$x_m^{\prime}$, respectively. Here, $x_n^{\prime} = x_n + {\delta}/{d_n}$, $x_m^{\prime} = b - {\delta}/{d_m}$ for a small positive  $\delta$ that ensures  $x_n^{\prime} < \epsilon_n$ and $x_n^{\prime}  < x_m^{\prime} $. The social welfare is increased because the model accuracy remains unchanged but the total cost is reduced, since
$d_n\left(f\left(x_n\right) - f\left(x_n^{\prime}\right)\right)-d_m\left(f\left(x_m^{\prime}\right)-f\left(x_m\right)\right)$
is greater than zero based on Lemma \ref{lem:con}. Thus, $\xb$ cannot be the socially optimal solution. 
\end{proof}
Based on Proposition \ref{them:sufficnt},  to maximize social welfare, clients  need to give priority to  reducing the noise rate of those clients' datasets with the highest noise rate. Meanwhile, Proposition \ref{them:sufficnt} inspires the efficient algorithm design for deriving the optimal solution in Section \ref{subsec:alg-swm}.

\subsection{Nash Equilibrium of Label Denoising Game}
We derive the NE of Game \ref{game:error} based on the best response of each client $n\in\N$. Specifically, let $\hne_n(x_n)\triangleq {\partial P_n(\xb)}/{\partial x_n}$ denote the partial derivative of the payoff function of client $n$ with respect to $x_n$. Thus, 
\begin{equation}\label{eq:best}
    \hne_n\left(x_n\right)  = \left(\frac{w_n}{N\bar{d}} \frac{\partial g\left(\bar{x}(\x)\right)}{\partial \bar{x}(\x)} - \frac{\partial f(x_n)}{\partial x_n}\right)d_n,
\end{equation}
which is strictly decreasing in $x_n$ due to the strict convexity of $f(\cdot)$ and the concavity of $g(\cdot)$. Let $\xzne_n $ denote the zero point of $\hne_n\left(x_n\right)$. The NE is determined as follows.
\begin{proposition}[NE]\label{lem:ne}
   The unique NE of Game \ref{game:error} is 
    \begin{equation}
\xne_n = \left\{\begin{array}{ll}
    \epsilon_n, & \rm{if }\; \text{$\hne_n \left(\epsilon_n\right) > 0$},  \\
    0, & \rm{if } \; \text{$\hne_n\left(0\right) < 0$} , \\
    \xzne_n, & \rm{otherwise}.
    \end{array}\right.
\end{equation}
As either $w_n$ increases or $\bar{d}$ decreases,  $\xne_n$ decreases. 
\end{proposition}
Proposition \ref{lem:ne} is proven based on the best response of the clients and the impact of system parameters on the zero point of $ \hne_n \left(x_n\right)$. The proof of uniqueness is given in an online technical report \cite{appendix}. In contrast with the optimal solution in Lemma \ref{lem:opt}, client $n$ cares about its own unit revenue, i.e., it reduces its noise rate under a larger $w_n$. In addition, similar to the optimal solution, if the average number of data samples decreases, client $n$ has the incentive to reduce its noise rate to improve the accuracy of the global model.


\subsection{Price of Stability in Model Accuracy}
We characterize the PoS\footnote{We use the term ``PoS" to refer to the performance reduction due to the self-interested clients. In Section \ref{sec:case}, we will follow the conventional mathematical definition of  PoS and define it  as the ratio of the maximum social welfare to the  social welfare under the unique NE in our work.} by showing that when compared wtih the optimal solution to problem \eqref{eq:swm},  the NE of Game \ref{game:error} leads to a higher average noise rate and hence a lower global model accuracy. Note that determining such an analytical result is not straightforward, because the socially optimal solution is not the solution that maximizes the model accuracy due to the fact that  higher accuracy is at the expense of higher correction costs for the clients. Specifically, 

\begin{lemma}[PoS in Label Denoising]\label{lem:x} The NE of Game \ref{game:error} leads to a higher average noise rate than the optimal solution to problem \eqref{eq:swm}, i.e.,
$\bar{x}(\xbne)\geq  \bar{x}(\xbswm)$.
\end{lemma}
\begin{proof}
    We prove this by contradiction. Suppose  $\bar{x}(\xbne)< \bar{x}(\xbswm)$. Thus, there exists an  $n\in\N$ such that $\xne_n<\xswm_n$. 
    Then, there exists a small positive $\delta$ that ensures  
    $\bar{x}(\xbne) + {d_n\delta}/(N\bar{d}) < \bar{x}(\xbswm) - {d_n\delta}/(N\bar{d})$ and $\xne_n + \delta < \xswm_n - \delta$. Based on the definition of $\xbne$ and $\xbswm$, 
    $w_n(g(\bar{x}(\xbne)) - g(\bar{x}(\xbne) + {d_n\delta}/{(N\bar{d})}))
        \geq d_n(f(\xne_n ) - f(\xne_n  + \delta))$ and $N\bar{w}(g(\bar{x}(\xbswm) - {d_n\delta}/{(N\bar{d})}) - g(\bar{x}(\xbswm)))  \leq d_n(f(\xbswm - \delta) - f(\xbswm))$, respectively.
    However, according to  Lemma \ref{lem:con},
    $g(\bar{x}(\xbswm) - {d_n\delta}/{(N\bar{d})}) - g(\bar{x}(\xbswm))  \geq g(\bar{x}(\xbne)) - g(\bar{x}(\xbne) + {d_n\delta}/{(N\bar{d})})$ and $f(\xne_n) - f(\xne_n + \delta)  > f(\xswm_n - \delta) - f(\xswm_n)$.
     This leads to a contradiction.
\end{proof}

Based on Lemma \ref{lem:x}, we derive the main result on accuracy. 
\begin{theorem}[PoS in Model Accuracy]\label{them:price}
     The socially optimal solution always leads to a higher accuracy of the global model than the NE does, i.e.,
    $A(\xbswm)\geq A(\xbne)$.
\end{theorem}

 Theorem \ref{them:price} implies that, when compared with  maximizing the social welfare, clients always obtain a global model with lower accuracy by considering their own interests at NE. 

Moreover, we present the PoS in social welfare for motivating our experiments in Section \ref{sec:case}. This result is obtained  based on the definition of social welfare maximization. 
\begin{lemma}[PoS in Social Welfare]\label{Lem: price} The socially optimal solution always leads to a higher social welfare than NE, i.e., 
    $\text{SW}(\xbswm)\geq \text{SW}(\xbne )$.
\end{lemma}





\section{Algorithm Design}
We propose algorithms for determining the optimal solution to problem \eqref{eq:swm} and NE, respectively.  These algorithms are applied for the case study in Section \ref{sec:case}. 
Designing a low-complexity algorithm for deriving the optimal solution is nontrivial due to the multi-variate decisions. Furthermore, we propose a best response algorithm to derive the NE.

\subsection{Social Welfare Maximization Algorithm}\label{subsec:alg-swm}
Based on Lemma \ref{lem:opt}, we can determine the optimal solution to problem \eqref{eq:swm} by finding the threshold $\theta$. 
Let $x_n(\theta) \triangleq  \min\left(\theta, \epsilon_n\right)$, and  $\xb(\theta)\triangleq (x_n(\theta),n\in\N)$. Define a function $H\left(\theta\right) \triangleq \text{SW}\left(\xb(\theta)\right)$. Then, we  formulate a  problem with respect to $\theta$ as follows:
\begin{subequations}\label{eq:theta}
    \begin{align}
        \underset{\theta}{\text{maximize}}
        & ~~~ H\left(\theta\right)  \\
        \text{subject to}
        & ~~~ \theta \in [0, \max_{n \in \N}\epsilon_n].
    \end{align}
\end{subequations}
Given the optimal solution $\theta^{\textsc{swm}}$ to problem \eqref{eq:theta}, we can obtain the optimal solution to problem \eqref{eq:swm} based on \eqref{eq:theta1}. 

\begin{algorithm}[!t]
\setstretch{1}
 \caption{Social Welfare Maximization Algorithm} \label{alg:optimal}
        \begin{algorithmic}[1]
        \State $\theta^{\textsc{swm}}$ $\leftarrow$ determine the value of $\textstyle \theta \in [0, \max_{n \in \N}\epsilon_n]$ that maximizes  $H(\theta)$ using ternary search;
        \item  $\xswm_n$ $\leftarrow$ given $\theta^{\textsc{swm}}$, compute $x_n$ using \eqref{eq:theta1}  for  $n\in\N$;
        \end{algorithmic}  
    \end{algorithm}
    
To solve problem \eqref{eq:theta}, we derive the property of  $H\left(\theta\right)$.
\begin{lemma}[Function Property]\label{lem:hh} $H\left(\theta\right)$ is a unimodal function, i.e., $H\left(\theta\right)$ is monotonically increasing up to a certain value of $\theta$ and then monotonically decreasing.
\end{lemma}
The proof is given in the online technical report \cite{appendix}. Based on Lemma \ref{lem:hh}, since $H(\theta)$ is a unimodal function with respect to a single-dimensional variable $\theta$, we can find its maximum point by ternary search, based on which we can determine the optimal solution to problem \eqref{eq:swm} using Proposition \ref{them:sufficnt}. We present the algorithm in Algorithm \ref{alg:optimal}. 
The time complexity for obtaining optimal $\theta$ is $\mathcal{O}(N\log({1}/{\mu}))$, where $\mu$ denotes a pre-defined threshold that bounds the gap between algorithm output and ground truth values of $\theta$.

\subsection{Best Response Algorithm for NE}
According to \eqref{eq:best}, we propose a best response algorithm in Algorithm \ref{alg:best} to find the NE of Game \ref{game:error}. Sepcifically, given the decisions of other clients, each  client alternatively updates its noise rate to maximize its payoff. The algorithm terminates until convergence, i.e., the decision update gap of every client is smaller than a predefined Convg\_Thresh.

\begin{algorithm}[!t]
\setstretch{1}
 \caption{Best Response Algorithm} \label{alg:best}
        \begin{algorithmic}
             \State $ x_n \gets \epsilon_n,n\in\N$; $n\gets 0$; $\text{Convg\_Set} \gets \emptyset$;
             \While {$|\text{Convg\_Set}|< N $}
                \State $n\gets \text{mod}(n,N)+1$; $x_n' \gets$ zero point of $\hne_n(x_n)$;
                   \If {$|x_n - x_n'| > \text{Convg\_Thresh}$}
            		      \State $x_n\gets x_n'$; $\text{Convg\_Set} \gets \emptyset$;
            	     \EndIf
                  \State  $\text{Convg\_Set} \gets \text{Convg\_Set}\cup \{n\}$; 
                
             \EndWhile

        \end{algorithmic}  
    \end{algorithm}

\section{Case Study with Different Noise Types}\label{sec:case}
We present case studies to validate our theoretical analysis and draw new insights. We consider $N=5$ clients, e.g.,  a cross-silo FL scenario where organizations cooperate for training.  The clients perform FedAvg algorithm over MNIST dataset using CNN. 
Each client has $d_n=1.2\times 10^4$ data samples. We study two types of label noise \cite{gu2021instancedependent}:
\begin{itemize}
\item \textit{Random flipping}:  A client $n$'s label flips to an random incorrect label with probability $\epsilon_n$. It models the case where  data collection is subject to random errors.
\item \textit{Instance dependent}: Each data sample is relabeled based on the inference result of a pre-trained CNN model. This approach simulates the human annotation process, and the noise rates can be adjusted by changing the accuracy of the pre-trained CNN model. For example, if a CNN model has an accuracy $0.8$, then the error rate of a dataset whose labels are generated from the CNN model's inference results is around $0.2$.
\end{itemize}

\begin{figure}[t]
    \centering
    \includegraphics[height=4cm]{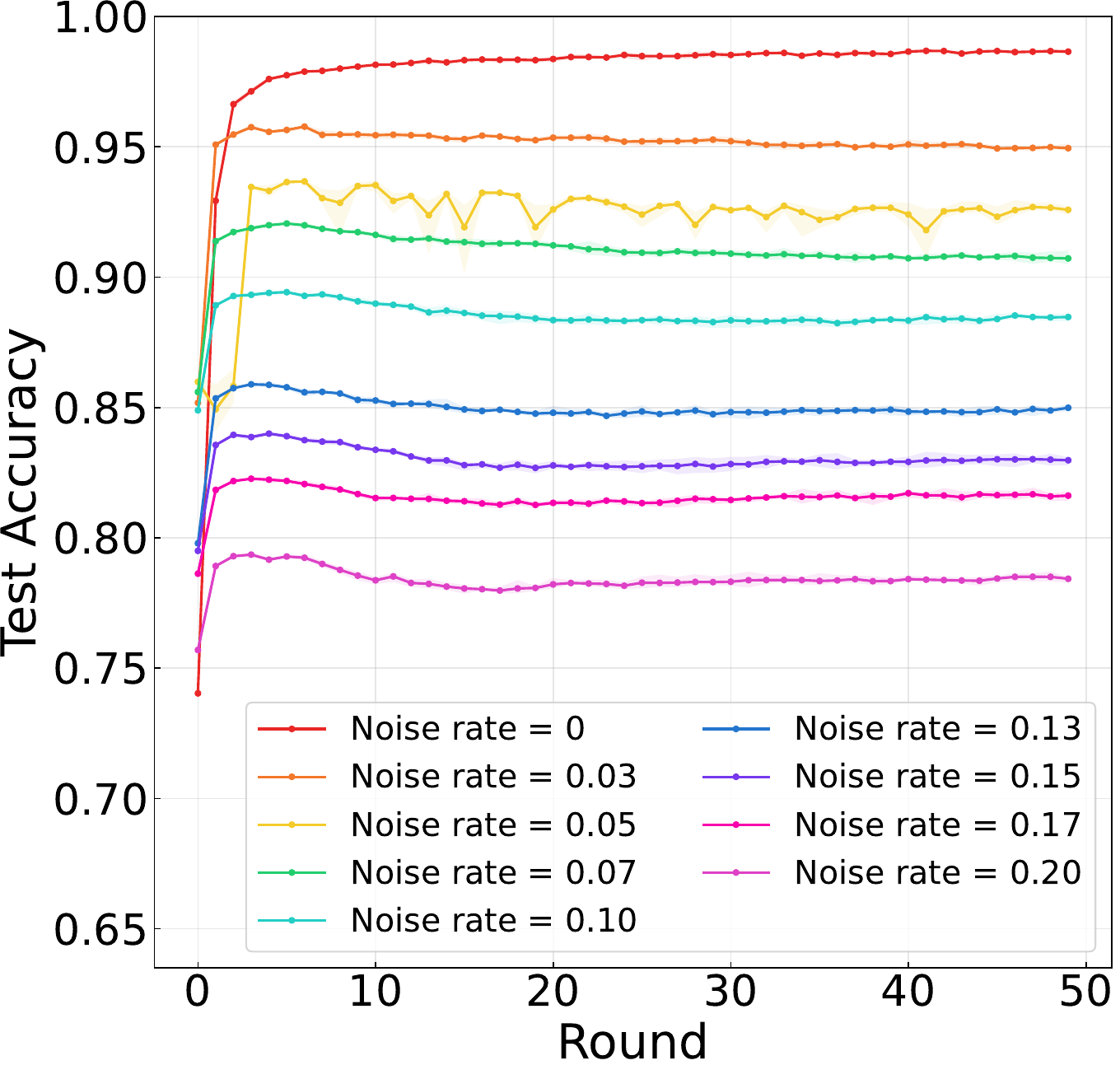}~\includegraphics[height=4cm]{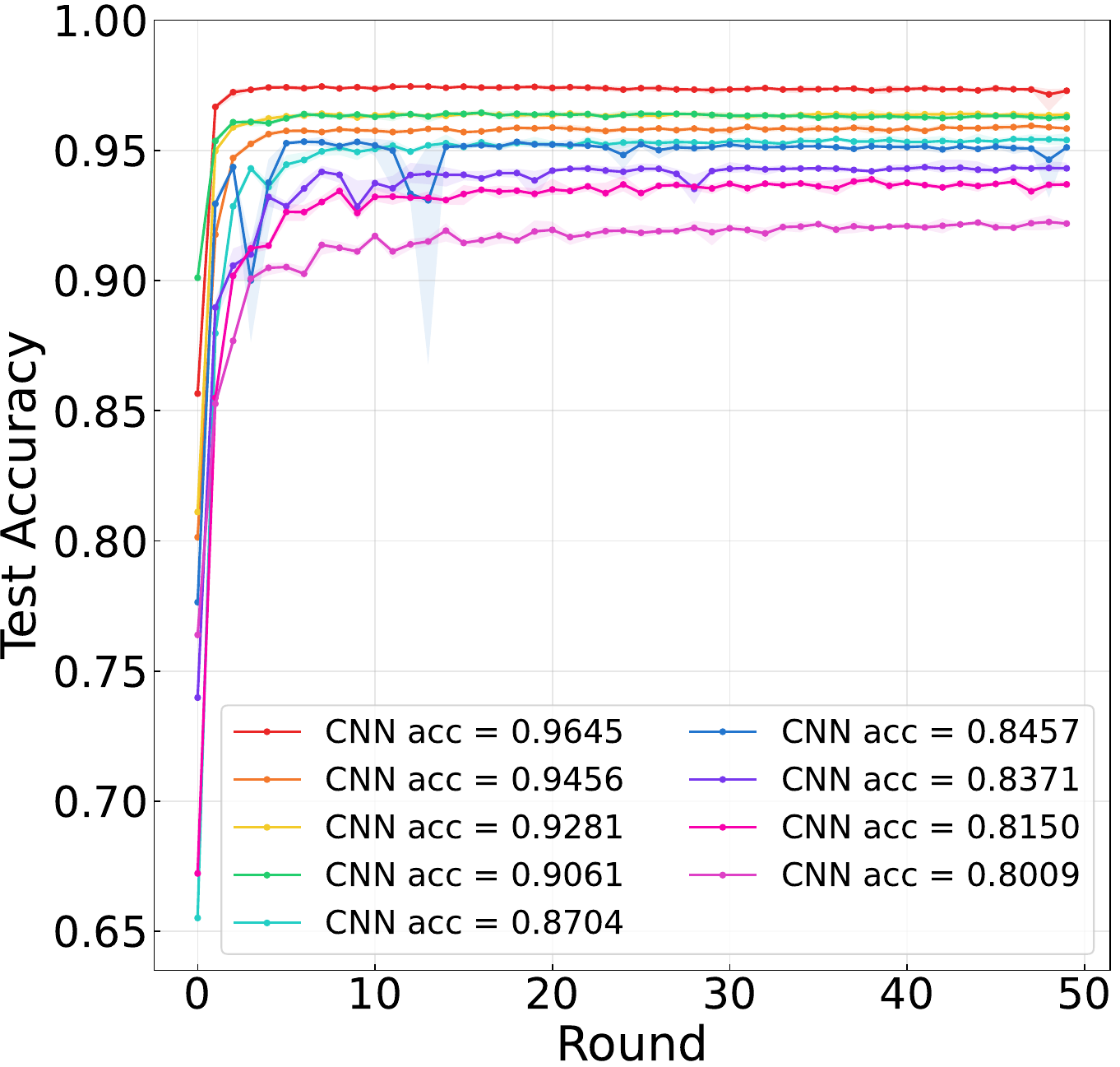}\\
    \qquad(a)\qquad\qquad\qquad\qquad\qquad\quad (b)\vspace{-2mm}
    \caption{Model accuracy under different noise rates: (a) random flipping; (b) instance dependency.}
    \label{fig:acc}
\end{figure}

\subsection{Impact of Label Noise on  FL}
We first study how label noise affects FL in terms of convergence and generalization. Fig. \ref{fig:acc} plots how the test accuracy of the global model changes with the training rounds (i.e., convergence). %
We observe that for both random flipping and instance-dependent noises, the existence of label noise slows FL convergence. Moreover, under random flipping, a large noise rate (e.g., $0.2$) causes overfitting.  However, overfitting is not obvious under instance-dependent.

\begin{figure}[t]
    \centering
    \includegraphics[height=3.5cm]{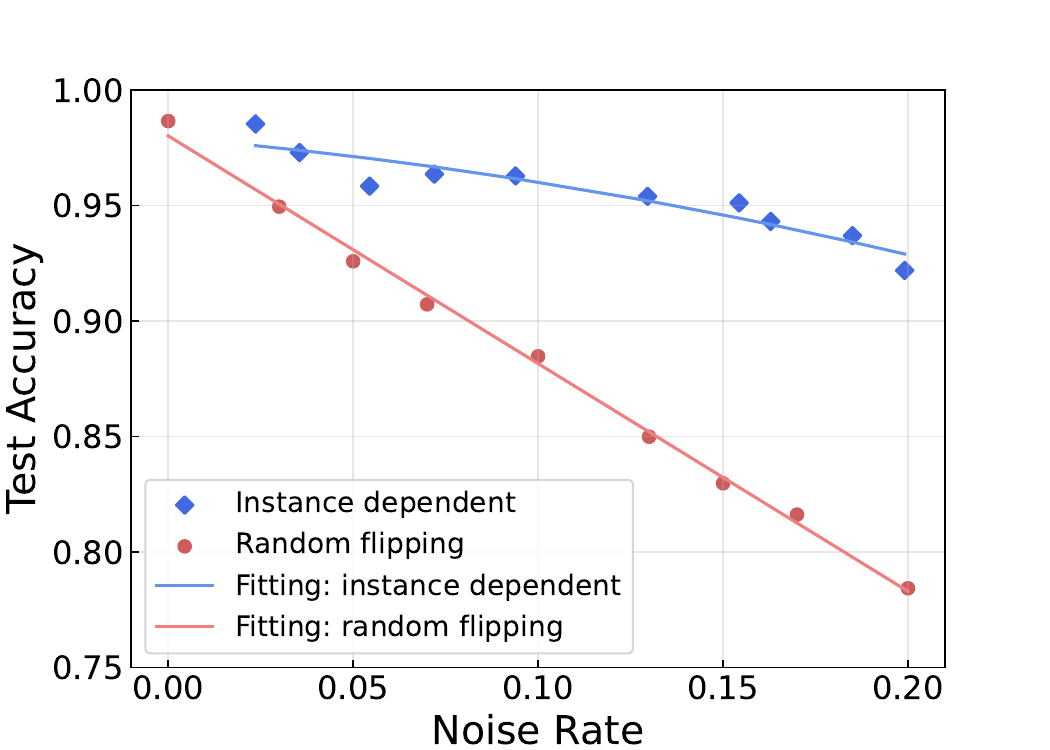}\vspace{-2mm}
    \caption{Model accuracy under random flipping and instance dependency.}
    \label{fig:error}
\end{figure}

Fig. \ref{fig:error} plots how the test accuracy changes with the noise rate at round $50$ (i.e., generalization). We observe that the global model performance decreases in the error rate. Interestingly, we observe that instance-dependent noise is less harmful to the model accuracy. Furthermore, we use curve fitting to simulate the impact of noise rate on global model accuracy and obtain the results below: 
\begin{itemize}
\item \textit{Random flipping}: The accuracy degradation across noise rate
$\epsilon$ fits a linear function. In addition, Table \ref{table:unchange}
shows that given the same average noise rate, the heterogeneous noise rates of the clients do not have a significant impact on the accuracy. Based on the above observation, we adopt a linear formulation of $g(\bar{x}(\x))$ with respect to $\bar{x}(\x)$ for random flipping:
\begin{equation}
    g\left(\bar{x}(\x)\right) = \kappa \bar{x}(\x)+g_0,
\end{equation}
where $\kappa < 0$ and $g_0\in(0,1]$ correspond to the accuracy decreasing per unit of noise rate and the  model accuracy under no label error, respectively.
\item \textit{Instance dependent}: The accuracy degradation across noise rate $\epsilon$ approximately fits a quadratic function:
\begin{equation}
    g\left(\bar{x}(\x)\right) = g_2\left(\bar{x}(\x)\right)^2+g_1\bar{x}(\x)+g_0,
\end{equation}
where $g_2 < 0$, $g_1 < 0$ and $g_0\in(0,1]$.
\end{itemize}
Note that the two fitted functions validate our assumption on $g(\cdot)$, which is a concave and decreasing function. In what follows, we will use the fitted functions to evaluate the PoS.

\subsection{Impact of Label Noise on PoS}

To study the PoS, we use the following function to help model the denoising cost (see (\ref{cost-function})): 
\begin{equation}\label{eq:cost}
    f(x) = c\ln{x},
\end{equation}
where $c < 0$. Formulation \eqref{eq:cost} is determined based on an error correction method called FedCorr \cite{xu2022fedcorr}.\footnote{In each round of training, each client rectifies the data samples with incorrect labels using the recent training models. Assuming that with the model obtained in each training round, a fixed fraction of noisy data samples can be found, then the correction cost is given by \eqref{eq:cost}.}
Let $\phi(\mu,\sigma,a,b)$ denote a truncated normal distribution with mean $\mu$, variance $\sigma$, and truncated range $[a,b]$. We consider $c = -1$,\footnote{Note that the relative values of $c$ and $w_n$ (rather than their absolute values) affect simulation results. We choose a unit cost for simplicity.} $w_n \sim \phi( 10^6,10^4,5\times10^5,1.5\times 10^6)$, $\epsilon_n \in \{0, 0.02, \cdots, 0.2\}$ for $n\in\N$.   Each experiment is repeated for 100 runs,  and we show the averaged results. Algorithms \ref{alg:optimal} and \ref{alg:best}  are used to compute the socially optimal solution and NE, respectively.

\begin{table}[t]
    \centering
    \small
    \caption{Model accuracy under heterogeneous noise rates.}\label{table:unchange}
    \begin{tabular}{l|c|c}
    \hline 
    Noise Rates $(\epsilon_n,n\in\N)$ & Average & Variance \\
         & Accuracy & Across Runs \\
        \hline 
        (0.1, 0.1, 0.1, 0.1, 0.1) & 0.888& 0.00044\\ 
        (0.05, 0.05, 0.1, 0.15, 0.15) & 0.889&0.00066\\ 
        (0.04, 0.08, 0.1, 0.12, 0.16) & 0.889&0.0005\\ 
    \hline 
    \end{tabular}
    \label{tab:my_label}
\end{table}
In Fig. \ref{fig:PoS}, we compare the  socially optimal solution (denoted by ``SWM") and NE in terms of  model accuracy and  the PoS in social welfare (i.e., the social welfare under SWM divided by that under NE). We observe that for both random flipping and instance-dependent noises, the model accuracy and social welfare under  ``SWM" are higher than that under NE.\footnote{In Figs.  \ref{fig:PoS}(c) and  \ref{fig:PoS}(d), a PoS that is larger than one implies a higher social welfare under SWM than that under NE.} This observation is consistent with Theorem \ref{them:price} and Lemma \ref{Lem: price}.  Importantly, in Fig. \ref{fig:PoS}(b), the PoS in social welfare  is generally larger under higher noise rate. That is, the non-cooperative behavior of clients imposes a higher efficiency loss when they have noiser data. This calls for an effective incentive mechanism to reduce the gap between socially optimal solutions and the equilibrium outcomes. Perhaps counter-intuitively, in Fig. \ref{fig:PoS}(d), the impact of unit revenue $w$ on the PoS relies on  label noise. Under instance dependent, a larger unit revenue leads to a higher necessity for motivating cooperative behaviors among clients.

\begin{figure}[t]
    \centering
    \includegraphics[height=3.5cm]{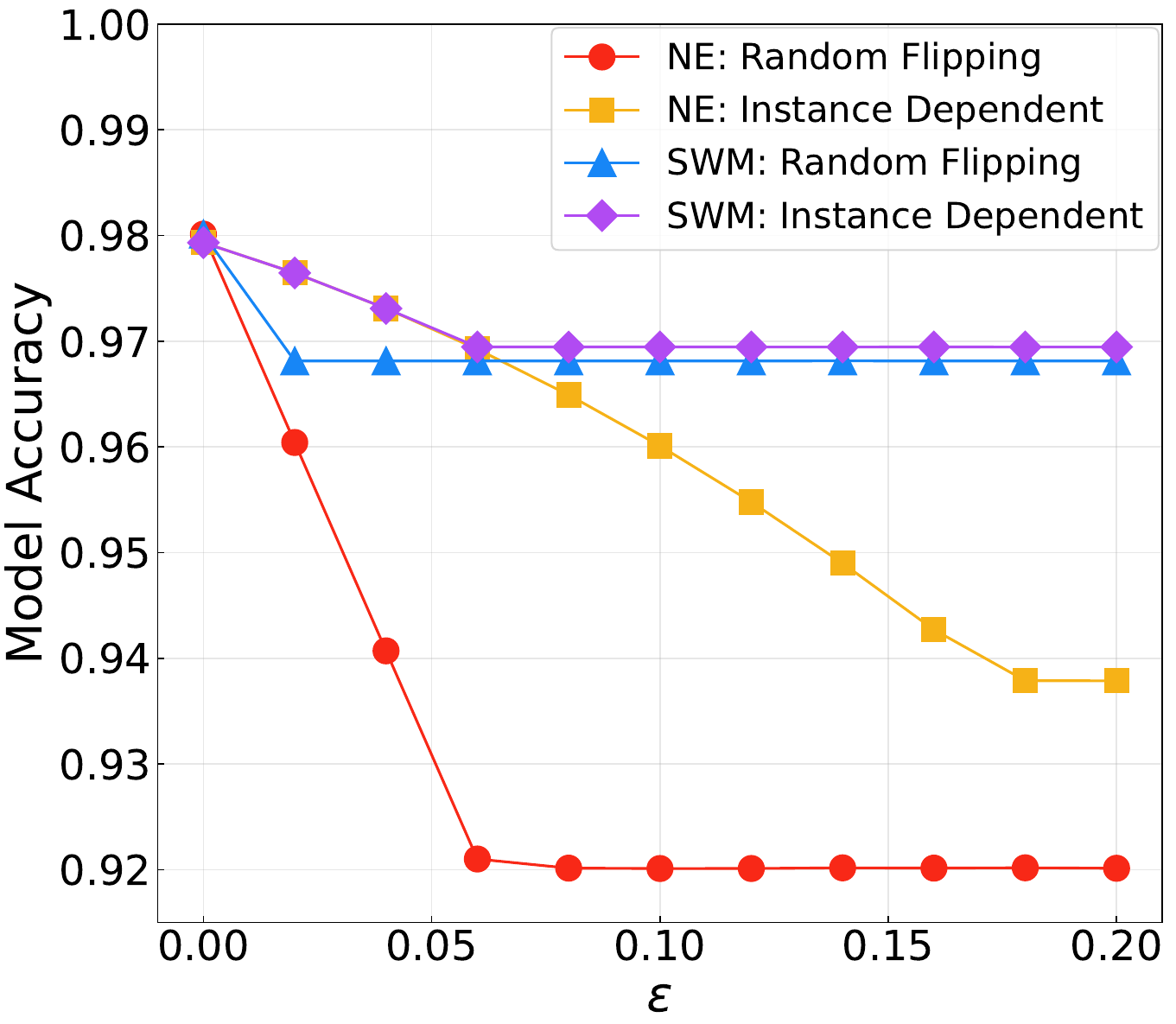}~\includegraphics[height=3.5cm]{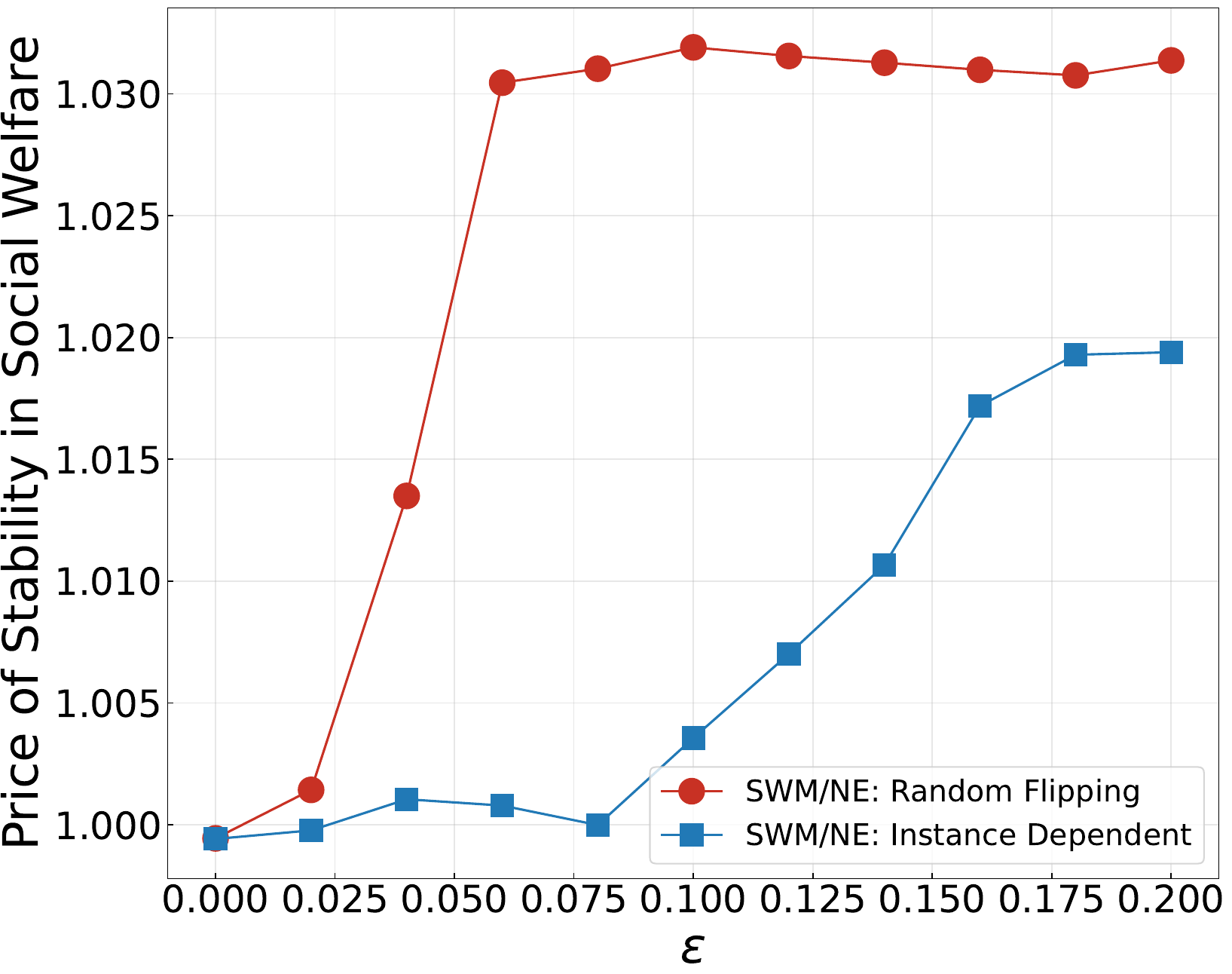}\vspace{-1mm}\\
    ~~~~(a)\qquad\qquad\qquad\qquad\qquad\quad (b)\\
     \includegraphics[height=3.5cm]{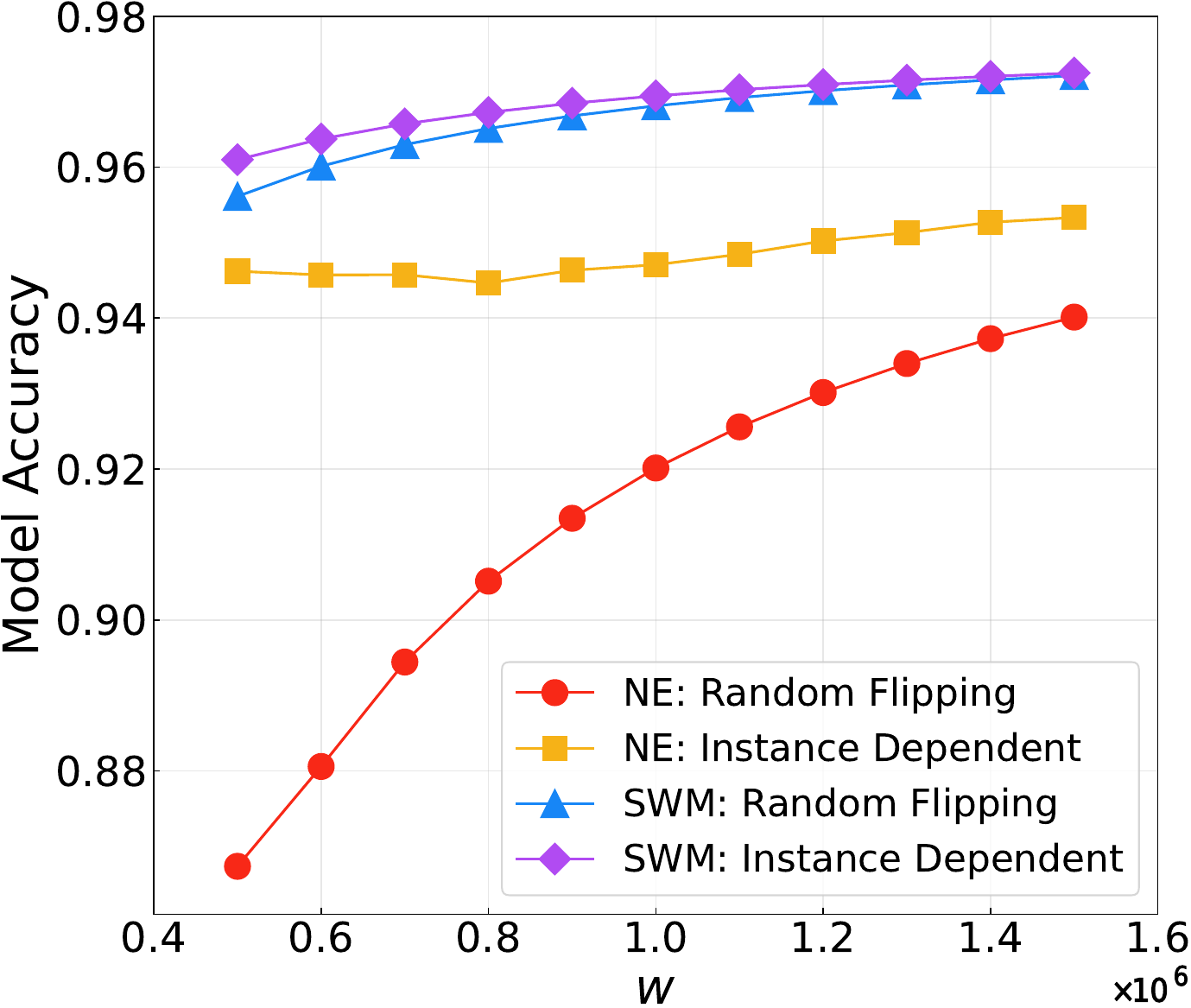}~\includegraphics[height=3.5cm]{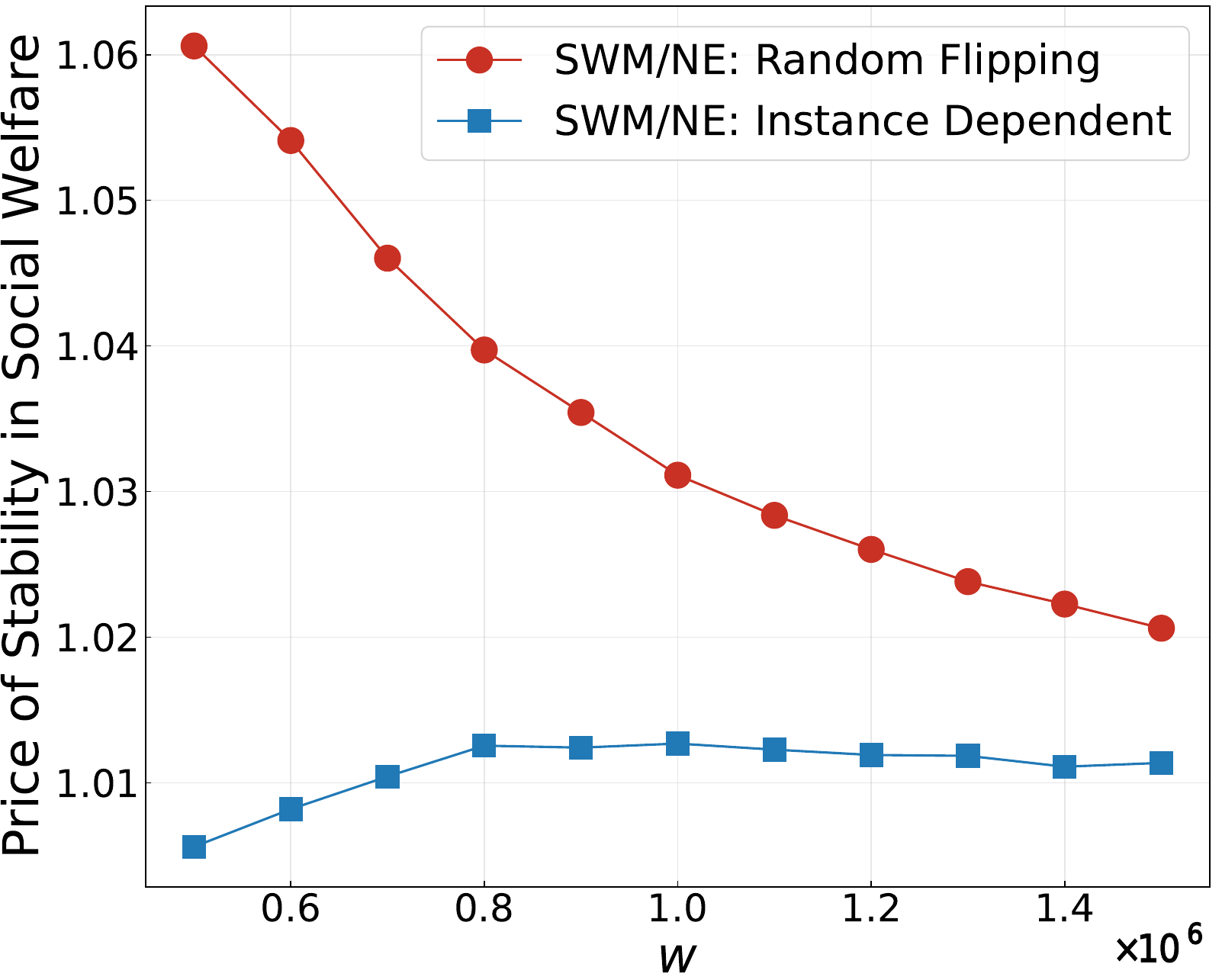}\vspace{-1mm}\\
    ~~~~(c)\qquad\qquad\qquad\qquad\qquad\quad (d)\vspace{-2mm}
    \caption{(a) Model accuracy and (b) PoS under different values of $\epsilon$ with $\epsilon_n=\epsilon$ for $n\in\mathcal{N}$; (c) model accuracy and (d) PoS under different values of $w$ with $w_n=w$ for $n\in\mathcal{N}$.}
    \label{fig:PoS}
\end{figure}

\section{Conclusion}
This paper presents a first study on the PoS in quality-aware federated learning. We formulate the clients' interactions as a novel label denoising game and characterize its equilibrium. We further derive the socially optimal solution, and find that the clients' non-cooperative behavior at NE always lead to a worse global model. Numerical results show that as the clients' data become noisier, the efficiency loss increases. This motivates the future work for an effective incentive mechanism as a remedy.

\bibliographystyle{IEEEtran}
\bibliography{Main-Chao}


\end{document}